\documentclass[10pt,twocolumn,letterpaper]{article}

\usepackage[pagenumbers]{cvpr} 

\usepackage{amsthm}  
\usepackage{thm-restate}
\usepackage{multirow}  
\usepackage{array}  

\definecolor{cvprblue}{rgb}{0.21,0.49,0.74}
\usepackage[pagebackref,breaklinks,colorlinks,allcolors=cvprblue]{hyperref}

\usepackage{bbm}

\newcommand{\iidd}{\ \mathop{\sim}\limits^{iid}\ }

\newcommand{\Reals}{\mathbb{R}}

\newcommand{\Exp}{\mathop{\mathbb{E}}}
\newcommand{\Prob}{\mathop{\mathbb{P}}}
\newcommand{\One}[1]{\mathbbm{1} \Big[{#1}\Big]}

\renewcommand{\vec}[1]{\boldsymbol{#1}}
\newcommand{\Dphi}{\mathcal{D}_\phi}

\DeclareMathOperator{\sign}{sign}
\DeclareMathOperator*{\argmax}{arg\,max}

\newcommand{\landauO}{O}

\newcommand{\cra}{certified robust accuracy\xspace}

\newcommand{\onn}{1-nearest neighbor\xspace}

\newcommand{\comment}[1]{}

\newcommand{\Xref}[1]{\Cref{#1}}

\newtheorem{proofsketch}{Proof Sketch}

\renewcommand{\paragraph}[1]{\medskip\noindent\textbf{#1}}

\def\paperID{30} 
\def\confName{CVPR}
\def\confYear{2025}

\newcommand{\papertitle}{Intriguing Properties of Robust Classification}

\title{\papertitle}

\author{Bernd Prach, Christoph H. Lampert \\
Institute of Science and Technology Austria (ISTA) \\
Klosterneuburg, Austria\\
{\tt\small bprach@ist.ac.at, chl@ist.ac.at}
}

\newcommand\blankfootnote[1]{%
	\begingroup
	\renewcommand\thefootnote{}\footnote{#1}%
	\addtocounter{footnote}{-1}%
	\endgroup
}

\begin{document}
\maketitle

\blankfootnote{© 2025 IEEE. Personal use of this material is permitted. Permission from IEEE must be 
	obtained for all other uses, in any current or future media, including 
	reprinting/republishing this material for advertising or promotional purposes, creating new 
	collective works, for resale or redistribution to servers or lists, or reuse of any copyrighted 
	component of this work in other works.}

\begin{abstract} 
    Despite extensive research since the community learned about \emph{adversarial examples} 10 years ago, we still do not know how to train high-accuracy classifiers that are guaranteed to be robust to small perturbations of their inputs. Previous works often argued that this might be because no classifier exists that is robust and accurate at the same time. However, in computer vision this assumption does not match reality where humans are usually accurate and robust on most tasks of interest. We offer an alternative explanation and show that in certain settings robust generalization is only possible with unrealistically large amounts of data. Specifically, we find a setting where a robust classifier exists, it is easy to learn an accurate classifier, yet it requires an exponential amount of data to learn a robust classifier. Based on this theoretical result, we evaluate the influence of the amount of training data on datasets such as CIFAR-10. Our findings indicate that the amount of training data is the main factor determining the robust performance. Furthermore we show that there are low magnitude directions in the data which are useful for non-robust generalization but are not available for robust classifiers. 
	We provide code at \url{https://github.com/berndprach/IntriguingProperties}.
\end{abstract}
    
\section{Introduction}
\label{sec:intro}

Deep learning has proven useful in numerous computer vision tasks, however, there are still shortcomings that come with these large end-to-end trained models. In particular, most state-of-the-art models suffer from adversarial examples \cite{szegedy_2014_intriguing}, which are tiny perturbations of the input that can result in large changes to the output of such models. This phenomenon can negatively affect the trust of users in the models, it might constitute a security issue, and --because it contradicts human experience-- it makes it impossible to create faithfully interpretable models. 

Research on mitigation strategies has concentrated on three pillars: In \emph{adversarial training}~\citep{szegedy_2014_intriguing,adversarial_training_2015_goodfellow} adversarial examples are created during training and the models are trained to classify those examples correctly. This procedure makes it harder to find adversarial examples, however, it cannot guarantee that no adversarial examples exist. In contrast, \emph{randomized smoothing}~\citep{randomized_smoothing_2019_cohen} acts at prediction time. It mitigates the effect of adversarial inputs by repeatedly evaluating the network, each time with different noise added to the input. It then constructs a final prediction by combining the predictions, \eg by means of a majority vote. This construction provides probabilistic robustness guarantees, however, usually several thousand predictions need to be performed for each input, which results in an undesirable slowdown. Finally, \emph{Lipschitz networks}~\citep{parseval_networks_2017_cisse} prevent adversarial examples by constraining the network architecture such that only models with a small Lipschitz constant (typically equal to $1$) can be learned. As a consequence, any input perturbation cannot cause a change in the network's output of larger magnitude than the perturbation itself, which yields deterministic and overhead-free guarantees on the presence of adversarial examples for any given input. This makes Lipschitz networks currently the most practical method for robust learning with guarantees. 

Unfortunately, despite many years of research, robust networks still achieve results far worse than what one might hope for. Even on fairly simple datasets and for fairly small perturbations the robust accuracy is much worse than what we believe is possible. Furthermore, a recent large study~\citep{compared_2024_prach} indicated that even architectures and training techniques that differ strongly in terms of their memory and computational demands, ultimately achieve quite similar robust accuracy values. This suggests the presence of a more fundamental barrier for the development of robust networks than what could be addressed by gradual improvements. Several explanations of this phenomenon have been put forward. For example, it has been suggested that there might be a natural trade-off between robustness and accuracy~\citep{tsipras_2019}: this would imply that high robust accuracy is just impossible to achieve. Alternatively, the hypothesis has been put forward that robust networks are not expressive enough \cite{fawzi_2018,odds_2019_nakkiran} or that the computational overhead of training robust network is the limiting factor \cite{degwekar_2019_computational,bubeck_2019_adversarial}. At the same time, there exist also recent works that do report that higher robust accuracy is achievable if \emph{additional training data} is exploited~\citep{gowal2021improving,wang2023better,altstidl2023raising,hu_2024_unlocking,hu_2024_recipe}. This would suggest that the problem is fundamentally one of \emph{generalization}.

Overall, however, we still lack a solid understanding of what makes the task of robust classification difficult. Our main contributions in this work are three \emph{insights} that we hope will clear up some misconceptions and hopefully guide future research on training robust networks in new directions.

\paragraph{Insight 1: There are settings in which learning robust accurate classifiers requires much more data than learning just accurate classifiers.}
Specifically, we present a learning problem in which any learning algorithm requires an amount of training data exponential in the data dimension, otherwise it cannot learn a robust classifier that is better than chance level.
Our construction is based on the fact that non-robust classifiers are able to exploit low-magnitude features in the data, while robust classifiers have to rely on high-magnitude features. The exponential gap between robust and non-robust learning opens up when the former generalize well but the latter ones do not.

\paragraph{Insight 2: Also on real data, the amount of data is a major determinant of performance.}
We provide evidence that the problem of Insight 1 is not just theoretical, but happens in (less drastic) form also for real datasets. Specifically, for MNIST, CIFAR-10 and CIFAR-100 we demonstrate that 
increasing the size of the training data reliably increases robust performance.
We further show that linear subspaces of the input space exist that only contain a tiny amount of the variance of the data, so those directions cannot be used for robust classification. However, when projecting our data into those subspaces, we can still obtain great (non-robust) accuracy. This implies that enforcing robustness makes classification a strictly hard task on CIFAR-10, providing an explanation why robust classification requires more data.

\paragraph{Insight 3: Robust architectures can fit and generalize non-robustly}. 
Training robust models requires certain architectural choices that are different from standard networks. We show that this is not the reason for the lack of performance on test data. Architectures built for robust classification are expressive enough to robustly overfit the training data, and we can also learn classifiers that generalize well, we just struggle to learn robust classifiers that generalize well. 
%

\medskip

We believe these insights show how important the amount of training data is for robust classification. 
%
In the remainder of the paper, we state our insights more formally and report in detail on our theoretical and empirical findings.




\section{Background \& notation}
\subsection{Accuracy and Robust Accuracy}

Traditionally, in machine learning, our goal is to maximize the \emph{accuracy} of a classifier $f$. For a training or test set $S=\{(x_1,y_1),\dots,(x_n,y_n)\}$ it is given by
\begin{align}
    \operatorname{acc}(f) &= \frac{1}{n} \sum_{i=1}^{n}\One{f(x_i)=y_i}.
\end{align}
In contrast, in this paper we consider the accuracy on adversarial altered inputs. We want our classifiers to predict the correct class, even when an adversary is allowed to change the input by a small amount. In order to measure performance in this setting, 
we define the \emph{robust accuracy} of margin $\epsilon$ as
\begin{align}
     \operatorname{RA}(f) &= \frac{1}{n} \!\sum_{i=1}^{n} 
    \One{f(\tilde x)=y_i \quad  \forall \tilde x: \|\tilde x-x_i\|_2 \le \epsilon}   \!\!
    \label{eq:RA}
\end{align}
%
%
Generally, computing a network's robust accuracy is NP-hard~\citep{virmaux2018lipschitz}, and even approximations are hard to obtain~\citep{jordan2020exactly}.  Therefore, we usually put certain constraints on the classifier $f$. In this work, we will usually choose $f$ to be a 1-Lipschitz classifier.

We define a \emph{1-Lipschitz classifier} to be a function of the form 
\begin{align} \label{eq:ols-classifier}
    f(x)=\operatorname*{argmax}_{i=1,\dots,K}[g(x)]_i,
\end{align}
where $g$ is a 1-Lipschitz function with $K$-dimensional outputs and $[\cdot]_i$ denotes the $i$-th component of a vector.
Here, a function, layer or network $f$ is \emph{$1$-Lipschitz} if 
\begin{align}
    \|f(x) - f(y)\|_2 \le \|x-y\|_2
\end{align}
for all $x, y$,
where $\|\cdot\|_2$ denotes the Euclidean norm. 
%

\citet{tsuzuku_2018_lipschitz} proved that with 1-Lipschitz classifier, we can easily compute a robustness guarantee. 
For a 1-Lipschitz function $g$, and the classifier $f$ as defined in \Cref{eq:ols-classifier}, we have that 
the robust accuracy (\Cref{eq:RA}) is bounded below by the $\emph{certified robust accuracy}$, which we define as $\operatorname{CRA}(f) = $
\begin{align}
    \frac{1}{n}\!\sum_{i=1}^{n} \One{ 
        [g(x_i)]_{y_i} > \max_{c \neq y_i}[g(x_i)]_{y} + \sqrt{2}\epsilon
    }  
    .
    \label{eq:CRA}
\end{align}
%
Therefore, in this work we will use $\operatorname{CRA}$ as an efficient, yet conservative, proxy for a network's actual robust accuracy. Unless specified otherwise, we will use a perturbation radius $\epsilon=36/255$, as it is common in the literature. 

\subsection{1-Lipschitz networks}
In order to train a 1-Lipschitz classifiers, we require a way of parameterizing the 1-Lipschitz function $g$. We do this by parameterizing $g$ as a neural network, where every layer has the 1-Lipschitz property. There are many ways of creating 1-Lipschitz linear layers, in this work we will use two rescaling-based layers: AOL~\cite{aol_2022_prach} and CPL~\cite{CPL}. For details about those methods, 
see \eg \cite{compared_2024_prach}.
\section{Robust classification needs more data} \label{sec:theory}
We start our discussion by two observations.
First, deep learning has been so successful for many classical computer vision tasks 
that it has become a routine task to train classifiers on a training dataset 
in a way so that the classifier also performs well on unseen test data afterwards. 
Second, for many such tasks, classifiers of high \emph{robust} accuracy are provably possible.
Namely, the human visual system provides proof for this, as human perception is typically not just 
highly accurate (we use it to generate the ``ground truth'' of our datasets), but also robust, in the sense that it is unaffected by small perturbations of its input. 

It is tempting to assume that those two observations (classifiers learned from data generalize well, and high-accuracy robust classifiers do exist) imply that it is also possible to \emph{learn} a high-accuracy robust classifier. However, in the following we show that this conclusion does not hold.
%
Informally, we show that for any dataset size there exists a family of data distributions such that 
robust classification is possible, 
learning an accurate classifiers is easy, but
learning an accurate robust classifier is impossible.

Our result is formalized in the following Theorem.
\begin{restatable}[No Free Robustness]{theorem}{nofreerobustness} \label{thm:hardness}
For any dataset size $n$ there exists a family, $\mathcal{F}$, of binary classification problems 
such that the following 3 properties hold:
\begin{enumerate}
    \item For any $\mathcal{D} \in \mathcal{F}$, there exists a classifier with $100\%$ robust accuracy.
    \item There is a learning algorithm that for any $\mathcal{D} \in \mathcal{F}$ and $S\sim\mathcal{D}$ finds a (linear) classifier with $100\%$ test accuracy.
    \item For any learning algorithm, $\mathcal{L}$, on average over $\mathcal{D} \in \mathcal{F}$ and $S \stackrel{\iid}{\sim}\mathcal{D}$, the learned classifier $\mathcal{L}(S)$ achieves robust accuracy less than $51\%$ on $\mathcal{D}$.
\end{enumerate}




\end{restatable}

\begin{proof}
	This proof consist of an explicit construction of a family of data distributions that fulfills the three conditions.
	
	\paragraph{Defining $\mathcal{F}$.} We first need to define our family of classification problems, $\mathcal{F}$. For that, we first set the data dimension to ${d = \lceil \log_2 n \rceil + 7}$. We denote the set of all binary functions on the $(d-1)$-dimensional hypercube as $\Phi$,
	\begin{align}
		\Phi = \big\{\phi: \{\pm1\}^{d-1} \rightarrow \{\pm1\}\big\}.
	\end{align}
	Note that this is a very large set with size $|\Phi| = 2^{2^{d-1}}$. For every $\phi \in \Phi$ we will define a data distribution $\Dphi$, then our  family of distribution is given as
	\begin{align}
		\mathcal{F} = \{ \Dphi: \phi \in \Phi\}.
	\end{align}
	In order to sample a pair $(\vec{x}, y)$ from $\Dphi$, we will sample $x_i$ uniformly from $\{+1, -1\}$ for $i = 1, \dots, (d-1)$. Then we will set 
	$x_d=\delta \phi(x_1, \dots, x_{d-1})$ for some small scalar $\delta$, and $y = \sign(x_d)$. 
	%
	Here, $x_1,\dots,x_{d-1}$ are \emph{robust} (large magnitude) features. Their relation to the ground truth label $y$ is deterministic ($y=\phi(x_1,\dots,x_{d-1})$), but it is hard to learn because of the size of $\Phi$.
	In contrast, $x_d$ is a useful, non-robust feature: It is perfectly correlated with the label, but because of its small magnitude it can be easily perturbed.
	%
	
	\paragraph{Proof of statement 1.} 
	For any $\Dphi \in \mathcal{F}$ with associated mapping $\phi$, consider the 
	classifier $f(x_1,\dots,x_{d}) = \phi\big(\operatorname{sign}(x_1,\dots,x_{d-1})\big)$, 
	where the $\operatorname{sign}$-function is applied componentwise.
	$f$ is robust against any perturbations of size $\epsilon<1$, because any such perturbation of the robust features is undone by the $\operatorname{sign}$ function, and it has perfect accuracy, because it coincides with the labeling function $\phi$. 
	
	\paragraph{Proof of statement 2.} 
	Consider a learning algorithm that always outputs the classifier $f(x_1,\dots,x_d)=\operatorname{sign}(x_d)$. Then, because $y=\operatorname{sign}(x_d)$ holds for all data distributions, it follows that $f$ has perfect accuracy on future data. 
	
	\paragraph{Proof of statement 3.} 
	The third property requires a slightly longer proof, resembling the \emph{No Free Lunch} theorem,
	\eg~\citep[Theorem 5.1]{shalev2014understanding}. Intuitively, it is based on the fact that a robust classifier cannot rely on the value of feature $x_d$, because that can be set to $0$ by a perturbation of size $\delta$. 
	Furthermore, the functional relation between $(x_1,\dots,x_{d-1})$ and $y$, is hard to learn because of the size of $\Phi$.

	\newcommand{\atkx}{\vec{\tilde{x}}}
	\newcommand{\aate}{average adversarial test error\xspace}
	\newcommand{\Xrob}{X^\text{r}}
	\newcommand{\xrob}{\vec{x^\text{r}}}
	In order to proof the statement, 
	first note that the average adversarial test error for perturbation of size $\le \delta$ is given as 
	\begin{align}
		\Exp_{\Dphi \in \mathcal{F}} \
		\Exp_{S \iidd \Dphi} \
		\Prob_{(\vec{x},y) \sim \Dphi}
		\left[ \mathop{\exists}_{\vec{x'} \in \mathcal{N}_\delta(\vec{x})} 
		\text{ s.t. } 
		\mathcal{L}(S)(\vec{x'}) \ne y 
		\right],
    \end{align}
    for $\mathcal{N}_\delta(\vec{x}) = \{\vec{x'}: \|\vec{x} - \vec{x'}\|_2 \le \delta\}$.
    We can get a lower bound to this quantity by considering only a single attack that sets the ``non-robust'' feature $x_d$ to $0$. We will write $\atkx$ for the result of applying this attack to an input $\vec{x}$. With this we can lower bound the \aate by
	\begin{align}
		\Exp_{\Dphi \in \mathcal{F}} \
		\Exp_{S \iidd \Dphi} \
		\Prob_{(\vec{x},y) \sim \Dphi}
		\mathcal{L}(S)(\atkx) \ne y.
	\end{align}
	Next we will rewrite sampling $\Dphi \in \mathcal{F}$ and $S \iidd \Dphi$ as sampling $\phi \in \Phi$, and once we know $\phi$, sampling from $\Dphi$ is equal to sampling the robust features from the hypercube $\{\pm1\}^{d-1}$, as the non-robust feature $x_d$ and the label depend deterministically on the robust features. We can furthermore change the order of sampling the robust features and sampling $\phi \in \Phi$. We will write $\Xrob$ and $\xrob$ for these robust features. Using this, the lower bound on the \aate becomes: 
	\begin{align}
		\Exp_{\Xrob} \
		\Exp_{\xrob} \
		\Exp_{\phi \in \Phi} \
		\One{\mathcal{L}\big(\Xrob,\ \phi(\Xrob)\big)(\vec{\tilde{x}^\text{r}}) \ne \phi(\xrob)}
	\end{align}
	Now note that when $\xrob \not\in \Xrob$, then $\phi(\xrob)$ becomes independent of $\phi(\Xrob)$. Therefore, since we assumed a uniform distribution on $\Phi$, any learner will be correct exactly $\frac12$ of the time.
	When $\xrob \in \Xrob$, any reasonable learner should be able to predict correctly, we will bound the error in this case by $0$.
	Using this, the lower bound on the \aate becomes
	\begin{align}
		\Exp_{\Xrob} \ \Exp_{\xrob} \ \One{\xrob \not\in \Xrob} \frac12.
	\end{align}
	Now the probability that $\xrob \in \Xrob$ is at most $\frac{n}{2^{d-1}}$, so we know that the \aate is at least $\frac12 - \frac{n}{2^d}$, and therefore at least $49\%$ by our choice of ${d = \lceil \log_2 n \rceil + 7}$.
\end{proof}

\Cref{thm:hardness} establishes a lower bound on the worst case behavior, by showing that in certain settings we do require exponentially many data points (exponentially in the dataset dimension). Next, we provide a matching upper bound: As long as the input domain is bounded and some robust classifier exists, exponentially many data points suffice to learn an accurate and robust classifier. 
For this we do need to assume that there exists a robust classifier that is robust to perturbations with bounded $L_\infty$ norm. Note that here is the only part of the paper where we use $L_\infty$ norm, everywhere else we assume $L_2$ distances.
More precisely:

\begin{restatable}{theorem}{upperbound} \label{thm:upperbound}
    Assume that there exists a $L_\infty$ robust classifier (margin $\delta$) on data distribution $\mathcal{D}$, where the data points are in $[0, 1]^d$. Then as long as we have 
    $n \ge 37 \left\lceil \frac1\delta \right\rceil^d$ training points independently sampled from $\mathcal{D}$, 
    for margin $\delta/2$,
    we can achieve average $L_\infty$ robust test accuracy
    of at least $99\%$
    (average over sampling training sets).
\end{restatable}

\begin{proofsketch}
    We will prove that in the setting above, the \onn classifier will get $99\%$ robust accuracy. In order to show this, we first show that for any test point, the nearest point of a different class is at least $2\delta$ away. Therefore, if there is a training point within distance $\delta$ of a test point, no perturbation of size at most $\delta/2$ applied to the test point can change the prediction of the \onn algorithm. Finally, we show that the probability (over sampling training set and test point) of having a training example within $\delta$ is $\ge 99\%$.
\end{proofsketch}


For the full proof see \Xref{sec:proof2}.
Note that we can adapt the proof to work for any margin $\delta' < \delta$, and not just for $\delta/2$. Furthermore, if we only assume $L_2$ robustness we might need many more data points, namely $\mathcal{O}(c^d d^{d/2})$ for some constant $c$.

\section{Experimental setup}\label{sec:setup}
In order to gather evidence towards quantifying the behavior of robust classifiers on datasets such as MNIST and CIFAR-10, we train some robust and standard (non-robust) models. We provide some background and describe architectures and training setup below.

\paragraph{SimpleConvNet.} In order to obtain simple models that achieve good accuracy we rely on the \emph{SimpleConvNet} \cite{SimpleConvNet_2024_Prach,wind_2022_cifar94}. It consists of $8$ convolutional layers and one linear layer. Each convolutional layer uses \emph{BatchNorm} \cite{ioffe_2015_batch_norm} and ReLU as the activation function. The model uses \emph{MaxPooling} in order to reduce the resolution in the forward pass and also as a global pooling before the linear layer. In order to calculate the loss, we first apply the \emph{Softmax} function with temperature $\frac18$ to the predicted class scores, and then use \emph{CrossEntropy}.

\paragraph{1-Lipschitz models.}
For the robust 1-Lipschitz models, we either use an 8-layer MLP or the ConvNet architecture from \cite{compared_2024_prach}. It constraints every single layer to be 1-Lipschitz, therefore the whole network is 1-Lipschitz as well. The architecture first concatenates channels with value $0$ to the input, so that the total number of channels becomes $64$. Then it applies $5$ blocks with $3$ convolutional layers each followed by a 1-Lipschitz linear layer. 
As 1-Lipschitz linear layers we use AOL \cite{aol_2022_prach} or CPL \cite{CPL}. 
As common in 1-Lipschitz networks, the architecture uses MaxMin \cite{maxmin_2019_anil} as the activation function. As down sampling it uses \emph{PixelUnshuffle}. 
Unless mentioned otherwise, the loss function we use is \emph{OffsetCrossEntropy} \cite{aol_2022_prach}, with offset and temperature both set to $\frac14$. The only hyperparameter we tune is the peak learning rate, we train models with different learning rates for $100$ epochs each, and pick the learning rate of the model with the highest \cra on a validation set.
For evaluation we usually train for 3000 epochs.

\paragraph{Randomized Smoothing} We also estimate the robust performance of a \emph{Randomized Smoothing} classifier \cite{randomized_smoothing_2019_cohen}. Based on a classifier $f: \Reals^d \rightarrow [C]$, for $C$ the number of classes, the \emph{smoothed classifier} $h$ is defined as
\begin{align}
	h(\vec{x}) = \argmax_{c \in [C]} \mathbb{P}(f(\vec{x}+\vec{\epsilon}) = c),
\end{align}
where $\vec{\epsilon}$ follows a certain multivariate Gaussian distribution: $\vec{\epsilon} \sim \mathcal{N}(\vec{0}, \sigma^2 I)$ for some fixed standard deviation $\sigma$. Suppose $f(\vec{x}+\vec{\epsilon})$ returns the two most likely classes with probabilities $p_1$ and $p_2$. Then, its smoothed classifier is robust to perturbations of size $\frac{\sigma}{2} \left( F_\text{G}^{-1}(p_1) - F_\text{G}^{-1}(p_2) \right)$, for $F_\text{G}^{-1}$ the inverse of a standard Gaussian cumulative distribution function.

Unfortunately, we cannot evaluate this classifier.\!~\footnote{When $f$ is a piecewise-affine classifier, we could in theory evaluate the probability. This is usually not possible in practice though.} However, we can approximate it by sampling $\vec{\epsilon}$. Furthermore, we can also estimate the robust performance of this (theoretical) classifier. Since in this chapter we are purely interested in estimating robust performance, we directly report this approximation in our results. 

We use a SimpleConvNet for the base classifier $f$ and during training we add Gaussian noise to the images in addition to the data augmentation described in \Cref{sec:optimization}. We found that setting the standard deviation $\sigma$ to $\frac18$ and training for $100$ epochs gave us good results on a validation set, therefore we used this values for our evaluation runs. We approximate the class probabilities by sampling $\vec{\epsilon}$ $1000$ times.

\paragraph{Optimization.} \label{sec:optimization}
We train all our models using SGD with Nesterov momentum of $0.9$ and batch size of $256$ with \emph{OneCycleLR} as a learning rate scheduler. As data pre-processing we subtract the training data mean from every channel, we do not rescale the data. We use the same data augmentation as \cite{wind_2022_cifar94,SimpleConvNet_2024_Prach}. It consists of random crops, random flips and setting a random patch of the image to zero.

\section{Experimental results}

\subsection{Robust scaling behavior} \label{sec:scaling-laws}

Recent work on robust image classification has shown that additional data can greatly increase robust accuracy on CIFAR-10. Also, in \Cref{sec:theory} we have shown that the amount of training data can be an important limiting factor for robust classification. Therefore, in this section, we want to explore how the size of the training data influences the performance of a robust classifier trained on real data.

\begin{figure}[t]
    \centering
    \includegraphics[width=0.9\linewidth]{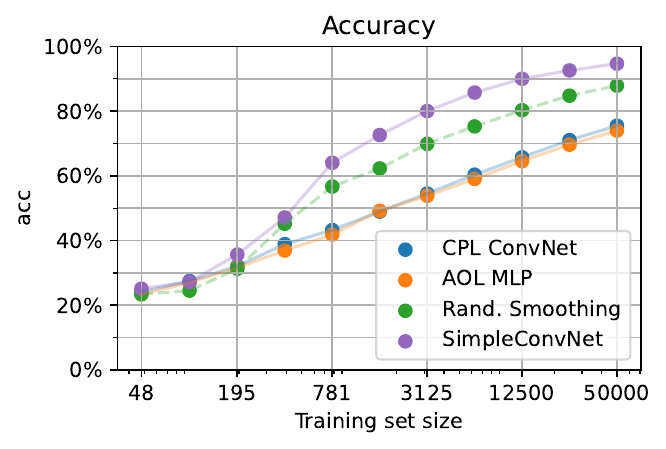}
    \includegraphics[width=0.9\linewidth]{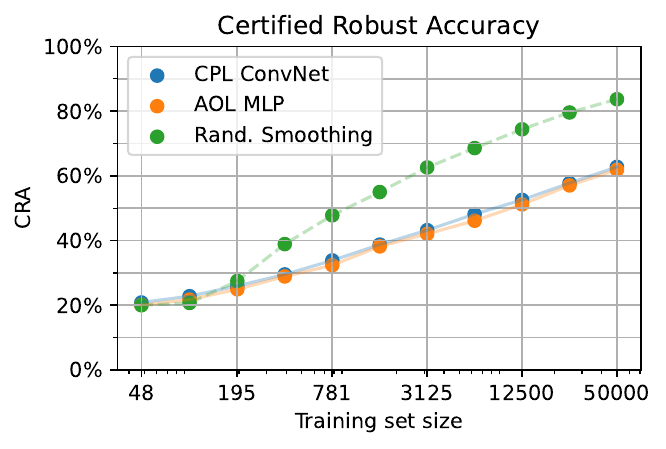}
    \caption{
        Accuracy (top) and \cra for ${\epsilon=36/255}$ (bottom).
        Training on subsets of CIFAR-10.
    }
    \label{fig:subsampling}
\end{figure}

%
We generated datasets of different sizes by sub-sampling the training partition of existing datasets such as CIFAR-10. We evaluated the performance of models trained on those smaller datasets.
%
In order to keep the amount of compute the same for all settings, when we divide the dataset size by some value $k$ we also multiply the number of epochs by $k$.
The results can be found in \Cref{fig:subsampling}.
We found that increasing the size of the dataset size does indeed make a big difference for the (test) performance of these models, and doubling the size of the dataset seems to reliably increase the \cra by about $5\%$. 

We repeated the experiments on MNIST and CIFAR-100, see \Cref{fig:scaling-behavior-other} for the results. Considering all results together, we can nicely see that the \cra follows a sigmoid curve. 
No matter how little data, we can always robustly classifier at chance level. This is visible on CIFAR-100, as the curve starts out almost flat for little training data. With more data, the \cra increases about linearly with the logarithm of the amount of training data. We can never get above $100\%$ \cra, so for larger dataset sizes the curve flattens again, which can be seen in the MNIST experiment.

\begin{figure}[t]
    \centering
    \includegraphics[width=0.9\linewidth]{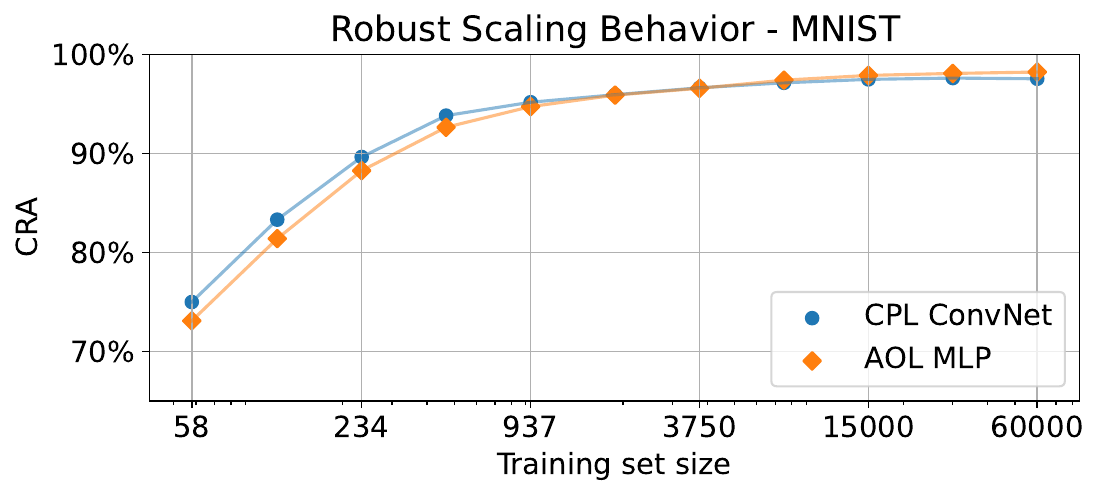}
    \includegraphics[width=0.9\linewidth]{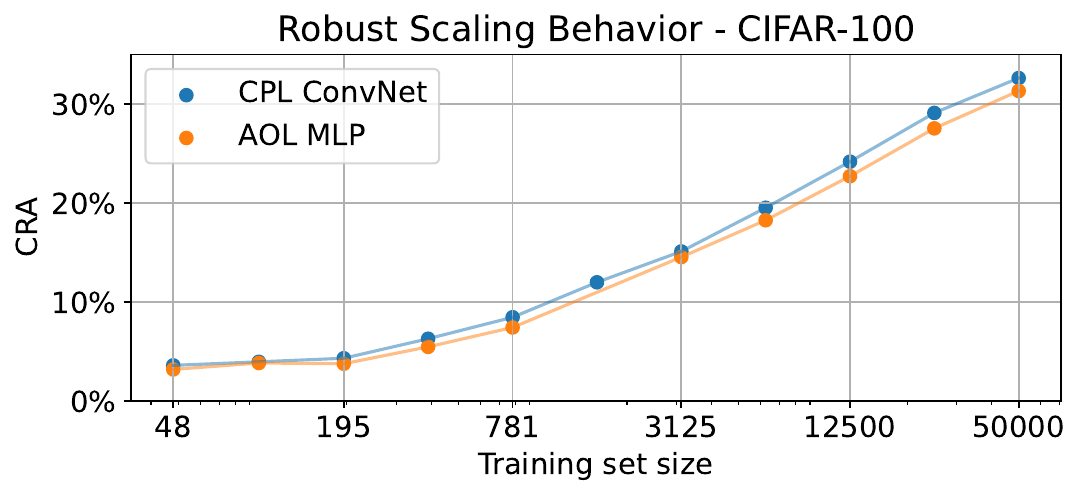} \caption{Scaling behavior on MNIST (top) and CIFAR-100 (bottom).}
    \label{fig:scaling-behavior-other}
\end{figure}

Interestingly, we also observe that across dataset, when training long enough, the convolutional architecture and the MLP have a very similar performance. It does seem that for robust classification, the inductive biases from the convolutional architecture are not very helpful, which is a big difference to non-robust classification.

We did a similar set of experiments to evaluate the influence of the amount of compute. The results are in \Xref{fig:scaling-compute}. Whilst increasing only the compute does also improve the performance, it has much less of an effect, and the curves flatten out with increasing compute.
%
Note that of course with more data, additional compute will be more useful. 
So we do expect that ideally we scale up both, dataset size as well as compute.

Recently, different pieces of work \cite{gowal2021improving,wang2023better,altstidl2023raising,hu_2024_unlocking,hu_2024_recipe} have used additional data in the style of CIFAR-10 generated by a diffusion model in order to improve the performance of robust classifiers. We also conduct scaling experiments on additional data, and find that the scaling nicely extends to larger sizes. See \Cref{sec:additional-scaling} for results.


\subsection{Robust and non-robust features} \label{sec:features}


In our theoretical section we have established that robust classification can be much harder than non-robust classification, which is also what we observe in experiments. In our theoretical example this hardness comes from 
a feature of small magnitude and high predictive power.
In this section we aim to evaluate whether CIFAR-10 has similar properties: We want to know whether there is a subspace of the input space, such that the data has very low variance when projected onto that subspace, yet the projection is useful for classification. 
It turns out that it is the case. For example, we can find such a subspace by considering the principal components \cite{pearson_1901_pca} of the dataset. The principal components of a data set are orthogonal basis vectors, such that principal component $i$ maximizes the variance of the data that lies in the subspace spanned by the first $i$ principal components. For visualizations of the first principal components of CIFAR-10 as well as the variance explained by subsets of principal components see \Xref{sec:x-features}.

For our experiments we flatten the training images in order to evaluate the principal components. Then in different experiments we project the flattened (train and test) images onto different subsets of principal components. After the projection, we transform the vectors back into the image space, so that we can train a standard and a 1-Lipschitz convolutional CPL network on the data without modifications to the setup.

\begin{figure}
    \centering
    \includegraphics[width=0.9\linewidth]{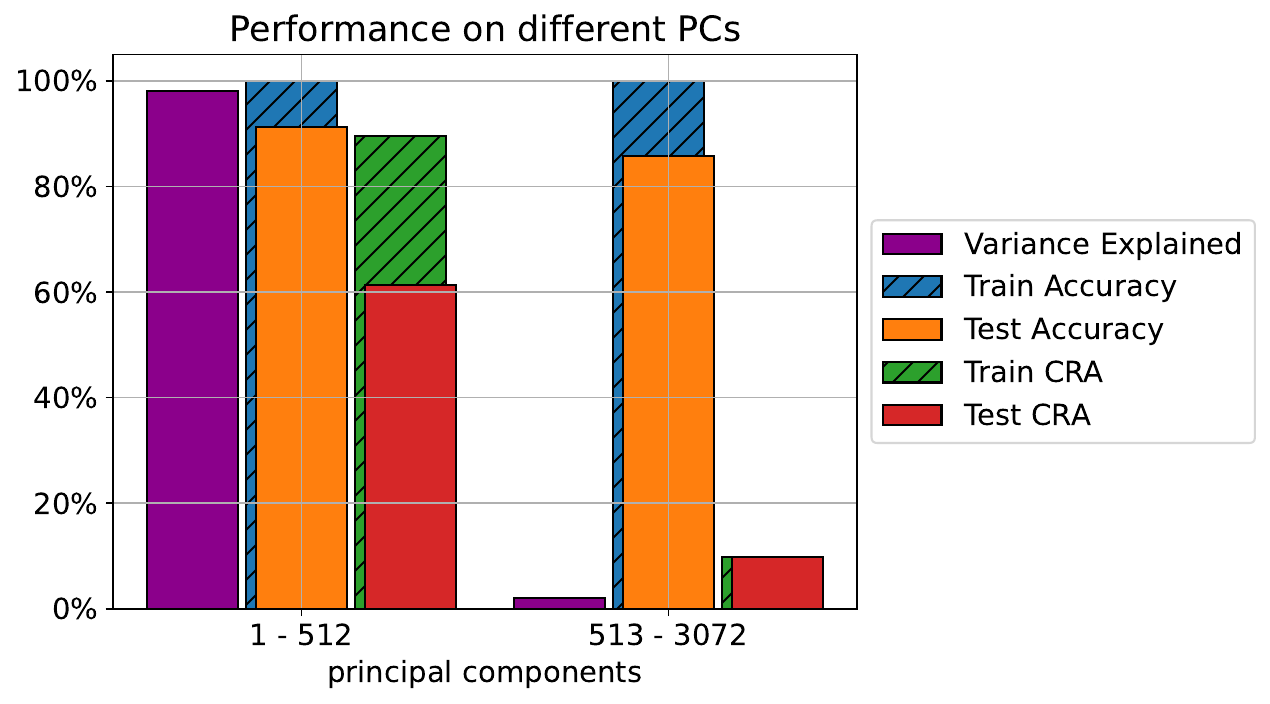}
    \caption{Performance on different subsets of the principal components.}
    \label{fig:pca-bars}
\end{figure}

Find our results in \Cref{tab:pca-results}, \Cref{fig:pca-bars} and \Cref{fig:pca-bars-extended}. 
%
It turns out that when considering the subspace spanned by all but the first 512 components, only about $2\%$ of the data variance are in this subspace. However, when training a standard network on this data, we obtain about $85\%$ test accuracy. Similarly, when projecting on the last 1024 (out of 3072) components, we get only $0.02\%$ of the variance, 
not enough to allow any robust classification, not even on the training data. However, we trained a standard network to achieve about $39\%$ accuracy on the test set, 
so there is still some weak signal in the last principal component, a signal that we cannot use for robust classification due to its low magnitude.

From this we conclude that there are in fact low magnitude directions on CIFAR-10 that are useful for classification, yet because of the small magnitude they are not useful for robust classification. This is a property that CIFAR-10 shares with our example from \Cref{sec:theory}.

We are also interested in the principal components of high variance.
It turns out that the first principal components are not very useful for generalizing. When using only the first 16 principal components for example, we do get about $72\%$ of the total variance,
and we can fit standard networks to get $\ge 99\%$ training accuracy. Furthermore we can also train 1-Lipschitz ConvNets to get good certified robust training accuracy 
($64\%$ with augmentation, $97\%$ without)
on this low-dimensional subspace.
%
%
However, the standard convolution network only achieves about $43\%$ accuracy on the test set, 
suggesting that there is only a weak signal in those features despite high variance.

Based on this observation, we created another dataset which contains the PCA components 1--16 together with the components 513--3072, that is, high magnitude and low magnitude features but not the intermediates. This data suffices for non-robust learning (86\% test accuracy), but robust learning fails (35\% robust test accuracy). We take this result as an indication that CIFAR-10 as a real dataset shares some characteristics with the hypercube example in \Cref{thm:hardness}: it contains high-magnitude features, which do not allow generalizing, and features of tiny magnitude, which generalize, but which no robust classifier is not able to exploit. 
%

\subsection{Robust overfitting} \label{sec:overfitting}
Previous works have suspected that the lack of robust performance might be due to underfitting: Our models might not have enough capacity to fit the data robustly. In our next experiment we will show that this is not the case, we can train a 1-Lipschitz networks to perfectly fit CIFAR-10, and do so in a robust way.
%
%
We train an CPL ConvNet and an AOL MLP.
In order to overfit robustly, we set the offset in our loss function to $\sqrt{2}$, and we train without augmentation for $3000$ epochs. 

We show the results in \Cref{fig:overfitting}. 
%
First note that we can clearly fit the training data robustly. For the MLP, even for perturbations of size $1$, we get almost perfect \cra on the training set. However, it is also visible that the classifiers do not generalize well. The performance on the test set is much worse for any perturbation size. 
Notably, robust overfitting does seem to require training for a lot of epochs. 

Importantly, this result does not imply that our models have enough capacity to fit the data distribution robustly, only that they have enough capacity for the amount of training data we currently have. We believe that when scaling up to amount of training data by a few magnitudes, of course we will benefit from larger models. However, the results in \Cref{fig:overfitting} shows that the capacity is not the current bottleneck.


\begin{figure}
    \centering
    \includegraphics[width=0.8\linewidth]{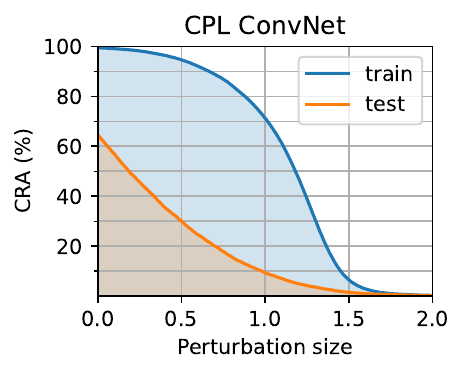}
    \includegraphics[width=0.8\linewidth]{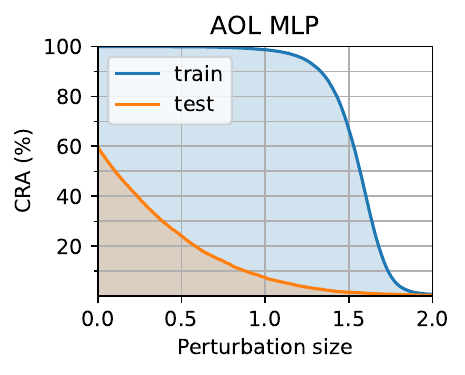}
    \caption{
        We can robustly overfit the CIFAR-10 training set with a CPL ConvNet (top) and an AOL MLP (bottom).
    }
    \label{fig:overfitting}
\end{figure}

\subsection{Robust architectures can generalize} \label{sec:trade-off}
In this final experimental section we want to explore whether it is the model architecture that prevents robust models from generalizing.
In order to prevent vanishing gradients and other problems when training 1-Lipschitz networks, there are a few important adaptation from standard convolutional networks.
Apart from the different linear layers, in 1-Lipschitz model we use different activation layers, no MaxPooling, no BatchNorm and different initialization strategies, see also \Cref{tab:arch-diff}.

In order to evaluate whether the difference in architecture (\eg the lack of global pooling in 1-Lipschitz networks) is the reason for the worse generalization, we carefully constructed a single architecture that can be trained to either competitive accuracy ($93.2\%$) or competitive \cra ($61.7\%$), depending on the choice of loss function.
%
%
This shows that 1-Lipschitz architectures can achieve compareable accuracy to traditional ConvNets, so we do not require layers such as MaxPooling for generalization.
We report details and results in \Cref{sec:x-rat}.


\section{Related work}


We will first describe related work on \textbf{robust generalization}.
One closely related piece of work is \cite{schmidt_2018_more_data}. In their work, the authors show that even in a very simple example on Gaussian distributions, robust classification can require many more training examples than non-robust classification. 
In their example, we can construct an accurate classifier from just a single training example, but for $\epsilon$-robust classification with $L_\infty$-norm we need about $\Omega(\epsilon^2\sqrt{d}/\log(d))$ examples, for $d$ the data dimension.\!\footnote{Here, $\Omega$ is the Bachmann–Landau notation: for functions $f$ and $g$, we write $f(d) = \Omega(g(d))$ if for some $M$ and $d_0$ it holds that $|f(d)| \ge M |g(d)|$ for all $d \ge d_0$.}
%
%
%
%
%
%
This is an very interesting results, that could explain part of the gap in performance between standard classification and robust classification. 
Our paper differs in that we are mainly interested in $L_2$ robustness, 
and in our example we show that the gap in samples required could potentially be much larger, and we might require $\Omega(2^d)$ training examples. 

The work of \cite{schmidt_2018_more_data} has been extended by \cite{bhagoji2019lower} and \cite{dan2020sharp}, where the authors also consider different $L_p$ norms, and prove some bounds about the excess risk. We are more interested in distribution where the optimal robust classifier actually achieves $0$ risk, and furthermore we think that the convergence rate to this optimal risk is not that important, but the sample complexity of getting within (\eg) $1\%$ of the optimal risk is a more useful quantity to study.

Other works have also produced results about robust generalization. For example, in \cite{raghunathan_2019} the authors introduce a data distribution where adversarial training can hurt the test performance of a (regression) model. The authors also argue that we do not actually need labeled data in order to improve performance. As long as we have unlabeled data, we can use a standard network to produce pseudo-labels, 
and (as long as the data is not adversarial) those labels should be fairly accurate.
In \cite{min_2021} the authors show that for Gaussian data the test loss of a linear classifier might actually get worse with additional data, or might show some double descent behavior.
In \cite{bhattacharjee_2021} the authors consider data distribution where an perfectly accurate robust classifier exists. They show that in this scenario, learning a robust classifier with maximal possible margin can need $d$ times more samples. We show in our paper that robust classification can be hard not just when aiming for maximum possible margin, but also when the goal is robustness to smaller perturbations.

A different attempt of explaining the lack of performance of current robust models is by blaming it on the \textbf{robustness-accuracy trade-off} \cite{tsipras_2019}. It has been observed theoretically as well as empirically that on certain data distributions a classifier can be either very accurate or reasonably robust, but not both \cite{tsipras_2019,raghunathan_2020_understanding,zhang_2019,dohmatob_2018_limitations,bethune_2022_misconceptions}. Whilst this trade-off offers interesting insights in general, we think it is not the most promising way to study the lack of performance in image classification tasks, where a classifier that is both accurate and robust at the same time does exist.

Other authors argue that robust classification might require much more complex models, where complexity can refer to the hypothesis class \cite{fawzi_2018,odds_2019_nakkiran,prach_2023_nact}, the amount of compute required \cite{degwekar_2019_computational,bubeck_2019_adversarial} or the size of the model required \cite{bubeck_2021,li_2022_robust}. 
The distributions used to prove results are often similar to our example distribution, there is a map that is hard to learn (\eg from the hypercube to the label), but the data comes with an additional feature of small magnitude that allows us read off the label. 
These results are further related to our work as in order to use exponentially (in $d$) many examples, one definitely also requires compute exponentially in $d$, at either training (\eg for a neural network) or at inference time (\eg for a 1-nearest neighbor classifier). So our result also implies the result that on certain distributions we do require exponential amount of compute.

Another related concept are \textbf{robust and non-robust features}  \cite{ilyas_2019_not_bugs_features}. 
The authors introduce the idea that
adversarial examples might not directly be artifacts of the way we train the models, but exist because of the data distribution. Furthermore, they exist because of features that are useful and generalize well but are not robust. 
In \cite{ilyas_2019_not_bugs_features} the authors use a very general definition of features, and consider any map from the input space to the real numbers a feature. 
We believe this definition is too general to give us insights about the datasets. 
We show that even linear subspaces of the input space exist with tiny variance, and yet projecting into these subspaces still allows achieving great (non-robust) performance.
%

There is also a recent piece of work \cite{bartoldson2024adversarial} that studied \textbf{robust scaling laws}, however they consider perturbations with bounded $L_\infty$ norm. In their setting they concluded that (with current techniques) it will require unreasonable amounts of compute 
(much more than to train recent LLMs) to match human performance on CIFAR-10. 

\section{Conclusion}
Even $10$ years after adversarial examples have entered the community's attention, robust classification is far from solved. Furthermore it is also not clear what makes the problem of robust classification so hard, and we still struggle on very simple datasets with robustness to fairly small perturbations.
%
%
In our paper we have aimed to collect theoretical facts and empirical evidence about robust classification, in particular about robust generalization, 
in order to
give the field a better understanding of the phenomena.

We first showed that there are data distributions on which it is not possible to train a robust classifier, unless the amount of data is unreasonably large. Moreover, this can be the case even for distribution where we can easily learn a good (non-robust) classifier, and where a perfect robust classifier exists.
%

Based on this insight, we evaluated whether similar results also hold on real data. We showed that on popular datasets including CIFAR-10, the performance of current models seems to be mainly determined by the amount of training data. Furthermore, we showed that as in our theoretical example, CIFAR-10 does have low-magnitude directions that cannot be used for robust classification, yet they are useful for training a standard classifier. 

Finally, we showed that the lack of performance of 1-Lipschitz classifiers is not a consequence of the architecture. In particular, 1-Lipschitz models are expressive enough to fit the training data very robustly. Furthermore, 1-Lipschitz architecture are also able to do (non-robust) generalization very well and the same architecture can be trained either to good test accuracy, or to good \cra based on the choice of loss function. We just currently fail do both (robust fitting and generalizing) at the same.

Overall we highlighted how important the amount of data is for training a robust classifier. 
We hope that future research will further explore scaling up datasets and classifiers, and that awareness of the intriguing properties of robust classification we presented will allow the community to create better robust image classifiers in the future.


{
    \small
    \bibliographystyle{ieeenat_fullname}
    \bibliography{references}
}


\clearpage

\setcounter{page}{1}
\maketitlesupplementary

\section{Proof of Theorem 2} \label{sec:proof2}
In this section we prove \Cref{thm:upperbound}. Recall
\upperbound*

\begin{proof}
    In order to prove this theorem we will show that in this setting, the 1-nearest neighbor algorithm achieves robust accuracy of at least $99\%$.
	
    In order to show this, we first show that for any test point, the nearest point of a different class is at least an $L_\infty$ distance of $2\delta$ away. Assume that $f$ is a robust classifier on $\Dphi$. Consider a test point $\vec{x}$ and the nearest training point $\vec{x_j}$ that is of a different class.  We know that $f$ robustly classifies both $\vec{x}$ and $\vec{x_j}$ with radius $\delta$. This implies that any point of distance $\le \delta$ to either of the points must share a label with that point, and therefore $\vec{x}$ and $\vec{x_j}$ must be at least $2\delta$ apart.
	
    For a test point $\vec{x}$ with label $y$, suppose there exists a training point $\vec{x_i}$ that is less than $\delta$ away from $\vec{x}$. By our assumption, this training point also has label $y$. Consider any other point $\vec{\tilde{x}}$ with $\|\vec{\tilde{x}} - \vec{x}\|_\infty \le \delta/2$. Furthermore, consider any training point $\vec{x_j}$ of a different class than $\vec{x}$. Using the triangle inequality we get that 
    \begin{align}
		&\|\vec{\tilde{x}} - \vec{x_i}\|_\infty 
		\le \|\vec{\tilde{x}} - \vec{x}\|_\infty + \|\vec{x} - \vec{x_i}\|_\infty
		< \frac{3\delta}{2}
		\\
		&\|\vec{\tilde{x}} - \vec{x_j}\|_\infty 
		\ge \|\vec{x} - \vec{x_j}\|_\infty - \|\vec{\tilde{x}} - \vec{x}\|_\infty
		\ge \frac{3\delta}{2}.
\end{align}
	Therefore, we know that the nearest training point to $\vec{\tilde{x}}$ must have label $y$. 
	This shows that the \onn is robust to perturbation of size $\le \delta/2$ of $\vec{x}$.

	With this established it just remains to be shown that with enough training examples,
	for $99\%$ of test points $\vec{x}$, there will be a training point close to $\vec{x}$.
	In order to prove that this is the case,  we will split the hypercube into a set of disjoint boxes. The probability of a test point being close to a training point is at least as big as the probability of the test point being in a box that has at least one training point inside.
	We define the boxes by defining a set of \emph{box centers}: For $D = \lceil 1/\delta \rceil$, set
	${\mathcal{C} = [\frac12 \delta, \frac32 \delta, \dots, \frac{2D-1}2\delta]^d}$.
	We define $B_r(C)$ as the $L_\infty$ ball with radius $r$ around $C$. 
	Further, we set $\mathcal{B}$ to be the set of all boxes, $\mathcal{B} = \{B_{\delta/2}(C): C \in \mathcal{C}\}$. We have that $|\mathcal{B}| = D^d = \lceil 1/\delta \rceil^d$.
	We will write $p_B$ for the probability of a data point lying in box $B$ under distibution $\mathcal{D}$. With this, we can write the probability of having at least $1$ training point in box $B$ as
	\begin{align}
		\Prob(\exists i: \vec{x_i} \in B)
		&= 1 - \Prob(\vec{x_i} \not\in B \ \forall i) \\
		&= 1 - \prod_{i=1}^{n} \left(1 - \Prob(\vec{x_i} \in B) \right) \\
		&= 1 - \left( 1 - p_B \right)^n.
	\end{align}
	We further have that $(1-p)^n = (1-p)^{\frac1pnp} \le \exp(-np)$, and therefore 
	\begin{align}
		\Prob(\exists i: \vec{x_i} \in B) \ge 1 - \exp(-n p_B).
	\end{align}
	
	We will also use that $\exp(x) \ge 1+x$ for all $x$, and therefore $\exp(x-1) \ge x$, and
	\begin{align}
		x\exp(-x) \le \exp(-1).
	\end{align}

	Then, putting everything together, for $p$ the probability that a test point $\vec{x}$ is classified robustly we have that:
	\newcommand{\pbox}{p_\text{box}}
	\begin{align}
		p
		&\ge \Prob\left(\min_i \| \vec{x} - \vec{x_i} \|_\infty \le \delta \right) \\
		&\ge \sum_{B\in\mathcal{B}} \Prob(\vec{x} \in B) \Prob(\exists i: \vec{x_i} \in B) \\
		&\ge \sum_{B\in\mathcal{B}} p_B (1 - \exp(-n p_B)) \\
		&= 1 - \sum_{B\in\mathcal{B}} p_B \exp(-n p_B) \\
		&= 1 - \frac1n \sum_{B\in\mathcal{B}} n p_B \exp(-n p_B) \\
		&\ge 1 - \frac{1}{n} \sum_{B\in\mathcal{B}} \frac{1}{e} \\
		&= 1 - \frac{|\mathcal{B}|}{ne}
	\end{align}
	Now since $|\mathcal{B}| = \lceil 1/\delta \rceil^d$ and we assumed that $n \ge 37 \lceil 1/\delta \rceil^d$ we get that $p \ge 99\%$.  
\end{proof}

Note that we can adapt the proof to work for any margin $\delta' < \delta$, and not just for $\delta/2$. Furthermore, if we only assume $L_2$ robustness we might need many more data points, namely $\mathcal{O}(c^d d^{d/2})$ for some constant $c$.

\section{Additional scaling law results} \label{sec:additional-scaling}
In this section we provide additional results for \Cref{sec:scaling-laws}.



\subsection{Scaling up compute}
First we explore the question of whether increasing the amount of compute alone can have a positive effect similar to the one we observed when increasing the size of the dataset. The answer seems to be no. We analyze for 3 different models how they scale with compute when leaving the dataset size fixed. The results are visualized in \Cref{fig:scaling-compute}. Note that the x-axis is in log scale. Doubling the dataset size improves the performance less and less, and the curves for both accuracy and \cra flatten with increasing training epochs.
\begin{figure}
    \centering
    \includegraphics[width=0.9\linewidth]{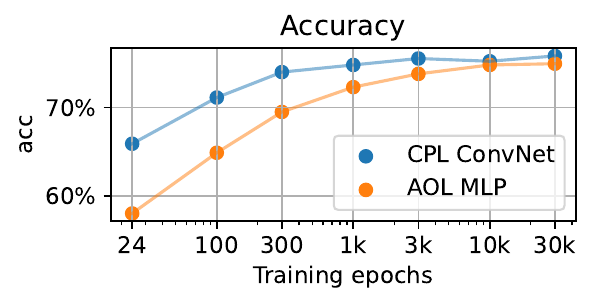}
    \includegraphics[width=0.9\linewidth]{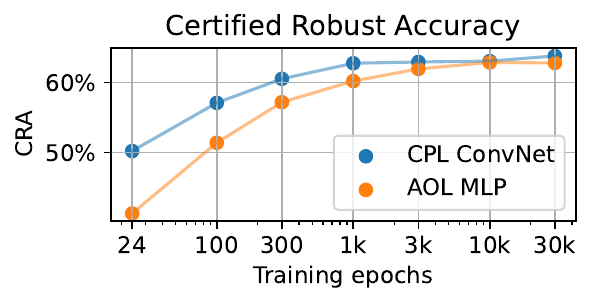}
    \caption{Scaling up the compute.}
    \label{fig:scaling-compute}
\end{figure}

\subsection{Training on additional data}
We are also interested whether the scaling law from \Cref{fig:subsampling} also extends to larger training datasets, larger than the $50k$ images from the CIFAR-10 training dataset itself. Recently, there has been a lot of works that used data generated in the style of CIFAR-10 by a diffusion model \cite{gowal2021improving,wang2023better,altstidl2023raising,hu_2024_unlocking,hu_2024_recipe}. 

We also followed this approach, and we used $1$ million images from \cite{wang2023better}. We subsampled this large dataset to obtain training sets of different sizes. For all experiments we set the number of epochs to $3000$, so we did use more compute with larger datasets this time.

In addition to our own results, we also report the performance of the current best 1-Lipschitz model \cite{hu_2024_recipe}. The authors generated $1$ million additional CIFAR-10 style images with a diffusion model, using a model trained on $940$ million images for data filtering. Training with this additional data allows them to achieve $78.1\%$ \cra.

We show the results in \Cref{fig:edm-scaling}. The scaling behavior from \Cref{fig:subsampling} extends to larger datasets sizes and generated data for all models considered. However, the improvements from training 1-Lipschitz models naively on more data, without increasing the model size eventually diminish. The results from related work that carefully designs the training setup to maximize test performance show that it is possible to extend the improvement. This effect might also in part be due to the lower quality of generated data.
For randomized smoothing our estimate of performance does scale very well with the amount of data.


\begin{figure}
    \centering
    \includegraphics[width=0.9\linewidth]{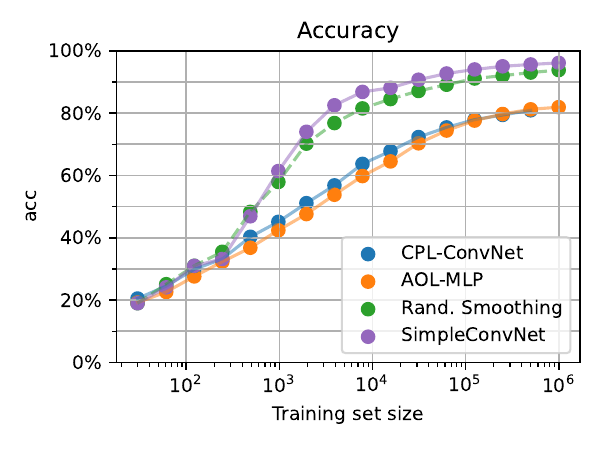}
    \includegraphics[width=0.9\linewidth]{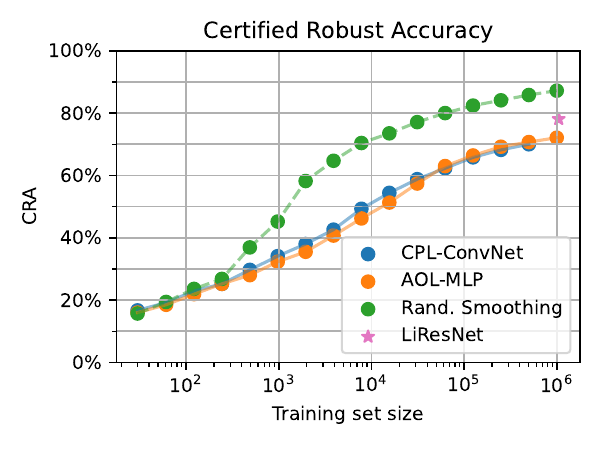}
    \caption{Scaling the size of the training data up by using additional data.}
    \label{fig:edm-scaling}
\end{figure}

\subsection{Setup for MNIST experiments}
Here we describe the setup we used for the additional results on MNIST in \Cref{fig:scaling-behavior-other}. In order to be able to keep most of our setup the same, we padded the image (after subtracting the mean value) with value $0$ to size $32\times32$. We reduced the size of our models slightly: we used $16$ instead of $64$ base channels for the \mbox{ConvNet}, and width $1024$ instead of $3072$ for the MLPs. We also simplified the augmentation a bit, and used only random cropping (size $4$) and random flipping.





\subsection{Why linear-log?}
It is very interesting that the performance in \Cref{fig:subsampling} increases approximately linearly in the logarithm of the training datasets size. Whilst we do not know why this is the case we do want to provide two intuitions to the reader.

First, if we assume that the performance of a classifier on a test point depends only on a constant number of "useful" examples (\eg the k-nearest neighbors), then the probability that an additional training example is "useful" for a test point is $\landauO(1/n)$. We also have that 
$\sum_{i=1}^{n} \frac{1}{i} \sim \ln(n) + \gamma$, where $\gamma\sim0.577$ is the Euler-Mascheroni constant,  
which might be the reason why we see this linear-log behavior.

\newcommand{\id}{d^*}
Second, when we consider the minimum distance of a test point to a training point, this distance should behave approximately proportional to $n^{-\frac{1}{\id}}$ for some $\id \le d$ (for details see \Cref{sec:linearlogx}). When $n \ll \exp(\id)$, this term is approximately equal to $1 - \frac{\log(n)}{\id}$,
so the distance to the nearest neighbor is approximately linear in $\log(n)$.
If our classifier improves (about linearly) as the nearest training examples get closer to the test points, the observation above would explain the scaling law we observe.
We provide more details as well as some experimental evidence in \Cref{sec:linearlogx}.

\section{Robust and non-robust features} \label{sec:x-features}
\newcommand{\lb}[1]{\makebox[8mm][r]{#1}}
\newcommand{\rb}[1]{\makebox[8mm][l]{#1}}
\begin{table}
    \centering

\begin{tabular}{c | >{$}c<{$} | >{$}c<{$}> {$}c<{$} | >{$}c<{$}> {$}c<{$}}
    \toprule
    PCs     & \multicolumn{1}{c|}{Var.}      & \multicolumn{2}{c|}{Accuracy \%} & \multicolumn{2}{c}{CRA \%}  \\
            & \multicolumn{1}{c|}{Expl. \%}     & \multicolumn{1}{c}{Train}         & \multicolumn{1}{c|}{Test}           & \multicolumn{1}{c}{Train}         & \multicolumn{1}{c}{Test}       \\
    \midrule
    \lb{1}-\rb{16}       & 72    & 99 & 43 & 64 & 31 \\
    \lb{1}-\rb{512}      & 98    & 100 & 91 & 89 & 61 \\
    \midrule
    \lb{513}-\rb{3072}   & 2     & 100 & 85 & 9 & 9 \\
    \lb{2049}-\rb{3072}  & 0.02  & 99 & 39 & 0 & 0 \\
    \midrule
    \lb{1}-\rb{16 \&}    & \multirow{2}{*}{$74$} & \multirow{2}{*}{$100$} & \multirow{2}{*}{$86$} & \multirow{2}{*}{$65$} & \multirow{2}{*}{$35$} \\
    \lb{513}-\rb{3072} &&&&& \\ %
    \midrule
    \lb{1}-\rb{3072}       & 100    &  100 & 94        & 93 &  62     \\
    \bottomrule
\end{tabular}

    \caption{Performance on different subsets of the principal components, as well as the proportion of variance explained by it. 
    }
    \label{tab:pca-results}
\end{table}

In this section we provide additional visualizations for \Cref{sec:features}.
%
For the performance on additional subsets of principal components see \Cref{tab:pca-results} and \Cref{fig:pca-bars-extended}.

%

In order to evaluate the capabilities of the models to overfit the training data projected onto different subsets of principal components, we repeated the experiments from \Cref{sec:features} without data augmentation. The results are shown in \Cref{fig:pca-bars-no-aug}. Note that we can robustly overfit the training data, even when projected on just the first $16$ principal components.

\begin{figure}
	\centering
	\includegraphics[width=0.9\linewidth]{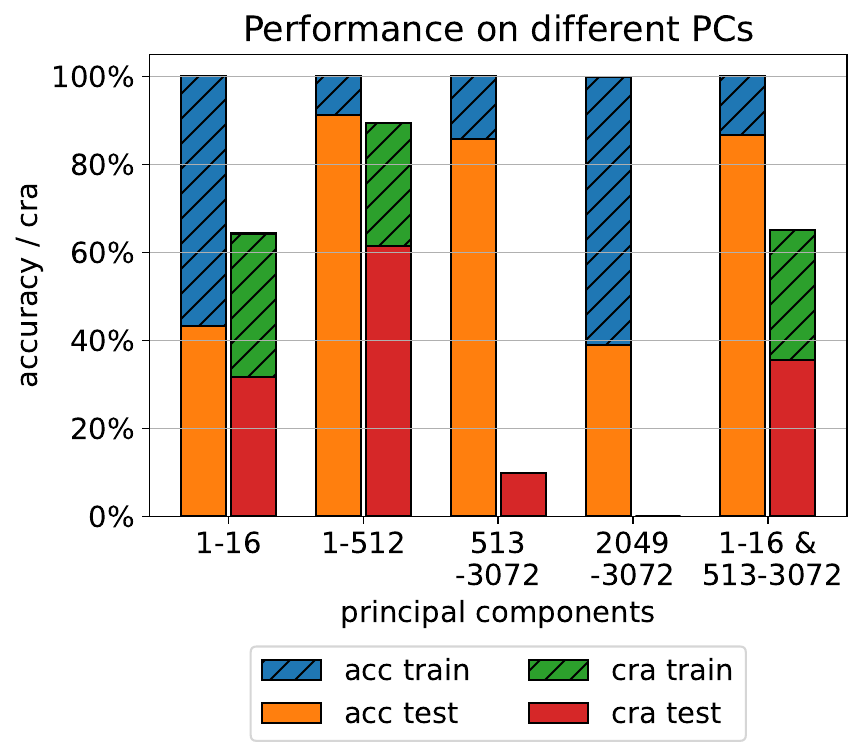}
	\caption{Performance on different subsets of the principal components.}
	\label{fig:pca-bars-extended}
\end{figure}

\begin{figure}
    \centering
    \includegraphics[width=0.9\linewidth]{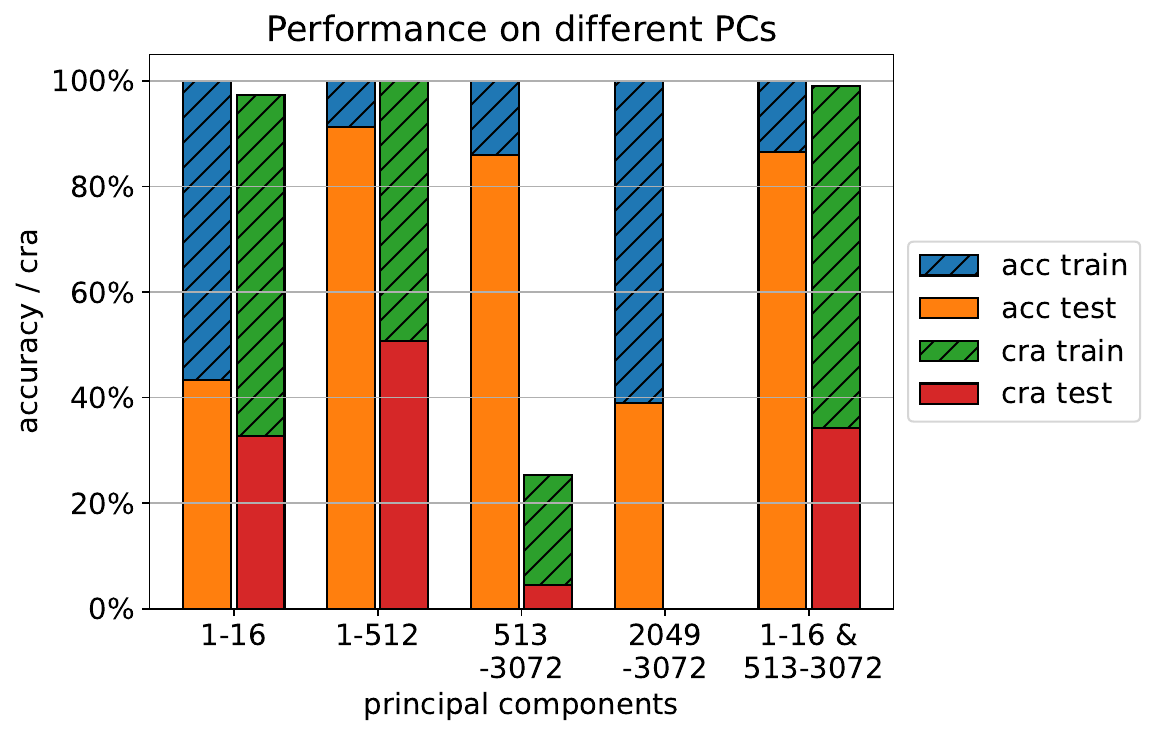}
    \caption{Performance of models when projected to a subset of principal components. Here, the robust models were trained without data augmentation.}
    \label{fig:pca-bars-no-aug}
\end{figure}

\section{Robustness-accuracy trade-off} \label{sec:x-rat}

In this final experimental section we want to explore whether it is the model architecture that prevents robust models from generalizing. Often the architecture, layers, and the training pipeline in general is different depending on whether accuracy or robust accuracy is the goal metric. Therefore, we want to explore whether the change in architecture is responsible for the lack of generalization in robust networks. 
In order to prevent vanishing gradients and other problems when training 1-Lipschitz networks, there are a few important adaptation from standard convolutional networks. 
For example, in 1-Lipschitz model we use different activation layers, not MaxPooling, no BatchNorm and different initialization strategies, see also \Cref{tab:arch-diff}.

In order to evaluate whether the difference in architecture (\eg the lack of global pooling in 1-Lipschitz networks) is the reason for the worse generalization, we carefully constructed a single architecture that can reach competitive accuracy and competitive \cra.

Among all differences between the architectures, we have found that initialization and batch normalization cause most of the a trade-off between accuracy and \cra, especially when training for a lower number of epochs.
For initialization, it seems that identity or near-identity initialization is very useful for 1-Lipschitz networks \cite{lot,compared_2024_prach}, whereas great accuracy requires some random initialization (\eg \emph{Kaiming uniform} \cite{he_2016_kaiming} or orthogonal). For the experiments in this section we used     orthogonal initialization.
Getting rid of the batch normalization is slightly trickier when using 1-Lipschitz layers. However, it turns out we can use a single normalization layer applied to the output of the model to enable training to good accuracy. We furthermore can fold this normalization into our loss function, so that we can train the identical model to good accuracy and (with a different loss function) to good \cra.

In order to smoothly interpolate between the two setting we introduce a loss function with a trade-off parameter $t$. It aims to be a version of the \emph{OffsetCrossEntropy} \cite{aol_2022_prach}, with additional normalization. We name the loss function \emph{SelfNormalizingCrossEntropy} 
and define it as:
\begin{align}
    \operatorname{CrossEntropy}\left(
    \operatorname{Softmax}\left(
        \frac{s}{\operatorname{std}(s) + t} - y
    \right),\ y \right),
\end{align}
where $s$ is the vector of scores predicted by the model, $y$ is a one-hot encoding of the label and $\operatorname{std}(s)$ denotes the standard deviation of $s$.

We used this loss function to train a set of models that includes ones with good accuracy and some with good \cra. For results see \Cref{fig:trade-off} and \Cref{tab:trade-off}.
%
Setting $t=0$ we can obtain $93.2\%$ accuracy with this model, showing that the model itself allows to generalize comparably to traditional ConvNets, 
and we do not require layers such as MaxPooling for generalisation.
With $t=\frac{1}{10}$ we can train the model to $61.7\%$ \cra.
This shows that we can train the same architecture to competitive accuracy or competitive \cra by only changing the loss function,
and therefore that the architectural restrictions of 1-Lipschitz models are not the reason why those models fail to generalize well. 

Interestingly, with our setup, there is no parameter value that is good for both tasks, but we do see a clear accuracy-robustness trade-off \cite{tsipras_2019}, as observed in the literature before.

\begin{table}
    \centering
    \begin{tabular}{c|cc}
        \toprule
        ConvNets:       & Standard          & 1-Lipschitz       \\
        \midrule
        Activation      & ReLU              & MaxMin            \\
        Blocks          & 3                 & 5                 \\
        Global Pooling  & MaxPooling        & None              \\
        Local Pooling   & MaxPooling        & PixelUnshuffle    \\
        Normalization   & BatchNorm         & None              \\
        Convolution     & Standard          & 1-Lipschitz       \\
        Linear Layer    & Standard          & 1-Lipschitz       \\
        Initialization  & Random            & Identity Map      \\
        \bottomrule
    \end{tabular}
    \caption{Difference in architecture of a SimpleConvNet and a standard 1-Lipschitz ConvNet.}
    \label{tab:arch-diff}
\end{table}

\begin{figure}
    \centering
    \includegraphics[width=0.9\linewidth]{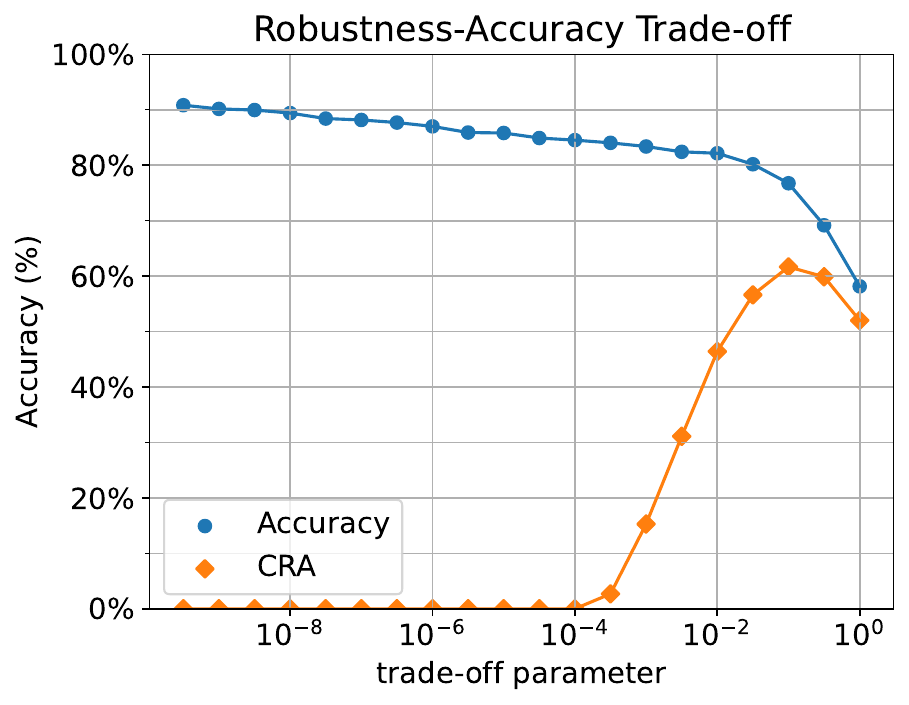}
    \caption{The same model can reach good accuracy as well as good \cra.}
    \label{fig:trade-off}
\end{figure}

\begin{table}
    \centering
    \begin{tabular}{c|cc}
        \toprule
        t & Acc & CRA \\
        \midrule
        0 & 93.2\% & 0.0\% \\
        $\frac{1}{10}$ & 76.8\% & 61.7\% \\
        \bottomrule
    \end{tabular}
    \caption{
        We can get high accuracy or good \cra with the same setup, only by changing the value of trade-off parameter $t$ in the loss function.
    }
    \label{tab:trade-off}
\end{table}

\section{Details for linear-log behavior} \label{sec:linearlogx}

In \Cref{fig:subsampling} it seems that the \cra depends on the logarithm of the size of the training set almost in a linear way. In this section we want to explore why this might be the case.

For a first intuition, we will consider the scenario where the performance of an architecture on every test example depends only on a few training examples. We think of those as \eg the $k$-nearest neighbors or some support vectors.
In this case, a new training data point can only have an influence on the performance (for a particular test point) if it is part of those few examples. The chance that this is the case for the $n^\text{nt}$ training example added to the training set is $\landauO(\frac1n)$. Therefore, the amount of times \eg a $k$-nearest neighbor classifier could be improved by adding an additional training example is $\landauO(\frac1n)$, and the total amount of times it might be improved when increasing the training set size from $n_1$ to $n_2$ is of order
\begin{align}
    \sum_{i=1+n_1}^{n_2} \frac1i \sim \ln(n_2) - \ln(n_1) = \ln\left(\frac{n2}{n1}\right).
\end{align}
If the amount of improvement does not increase with dataset size, this gives us an upper bound on the performance: For some value $c$, doubling the dataset size should at best increase the performance by $c$. This is in line with the behavior we observe in \Cref{fig:subsampling}. \medskip

We can also analyze this behavior in terms of distance to the nearest neighbor: We assume that for image datasets, for some dimension $d^*$ (something like an "intrinsic dimension of the data"), it should hold that the probability $p$ of a (test) data point being within radius $r$ of another data point approximately follows $p \sim c r^{d^*}$ for come value $c$.
We want to use this to make statements about the median of the distribution of the distance to the nearest neighbor, which we call $r^*$.
In order to do this, consider the probability $p_n$ that any of $n$ datapoints is close to the test point. We have that
\begin{align}
    p_n = 1 - (1-p)^n \le np
\end{align}
and 
\begin{align}
    p_n = 1 - (1-p)^n \ge 1 - \exp(-np).
\end{align}
Now if we set $r=r_l$ for
\begin{align} \label{eq:small-r}
    r_l = \left( \frac{1}{2cn} \right) ^ \frac{1}{d^*},
\end{align}
we get that $p_n \le np = nc \frac{1}{2nc} = \frac12$, and therefore $r^* \ge r_l$.
Similarly, for 
\begin{align} \label{eq:big-r}
    r_u = \left( \frac{1}{cn} \right) ^ \frac{1}{d^*},
\end{align}
we get that $p_n \ge 1 - \exp(-np) = 1 - \frac{1}{e} > \frac12$, and therefore $r^* \le r_u$.
Putting these together we have that
\begin{align} \label{eq:both-r}
    \left( \frac{1}{2cn} \right) ^ \frac{1}{d^*}
    \le r^*
    \le \left( \frac{1}{cn} \right) ^ \frac{1}{d^*}.
\end{align}

Furthermore note that as long as  $\log n \ll d^*$, the following approximation should be close:
\begin{align}
    \left( \frac1n \right)^{\frac1{d^*}} = \exp\left( - \frac{\log(n)}{d^*}\right)
    \sim 1 - \frac{\log(n)}{d^*}.
\end{align}
Therefore, for some $C$ it should approximately hold that 
\begin{align}
    C \left(1 - \frac{\log(n)}{d^*}\right) \le r^* \le C \left(1 - \frac{\log(2n)}{d^*}\right),
\end{align}
which does imply that $r^*$ will behave approximately linearly in $\log(n)$ as long as $\log(n) \ll d^*$



We evaluated whether this relationship does hold on CIFAR-10. Our results are shown in \Cref{fig:onn-distances}, where we show that indeed the distance to the nearest neighbor does behave similarly to what we expect from the theoretical analysis.
%
If it is further the case that the (expected) \cra of a classifier increases when a test point is closer to the training point, this would explain why the performance of this classifier scales about linearly with the logarithm of the dataset size. 
Note that at least when using the angular distance, it seems that on average 1-Lipschitz classifiers do better on test examples with a nearby training example.
%
%

When estimating $d^*$ from the experimental data in \Cref{fig:onn-distances}, we get an "intrinsic dimension" of about $d^* \sim 28$. This unfortunately implies that in order to get 1-nearest neighbors of distance close to $1$ we will require $n \ge 10^{31}$. So while the 1-nearest neighbor algorithm will produce a great robust classifier with enough data, this amount of data does not seem to be reachable in practice.

\begin{figure}
    \centering
    \includegraphics[width=0.9\linewidth]{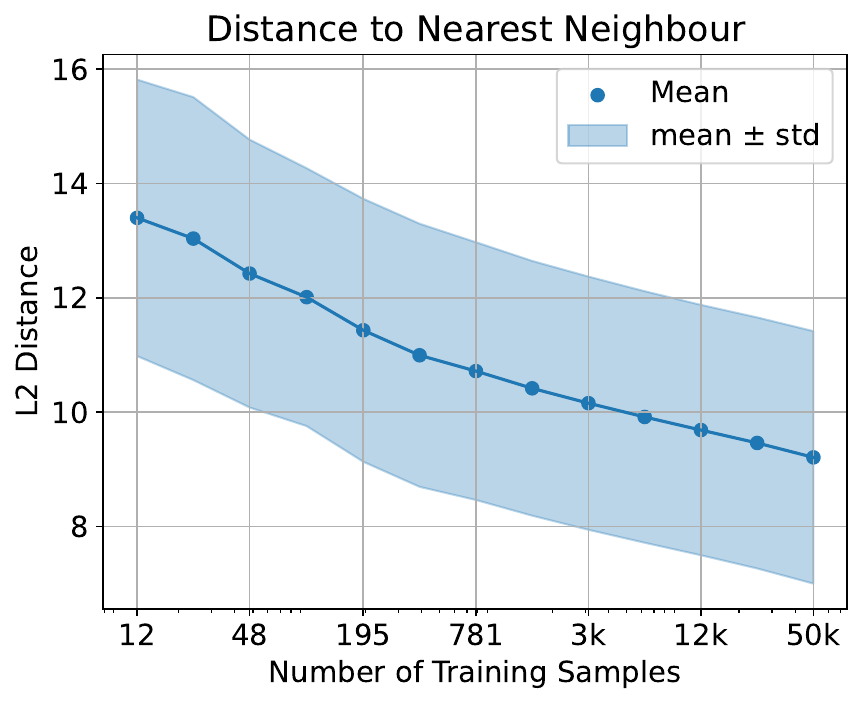}
    \caption{The distance to the nearest neighbor 
    scales about linearly in $\log(n)$, for $n$ the size of the training dataset.
    }
    \label{fig:onn-distances}
\end{figure}



\end{document}